\documentclass[12pt]{amsart}
\usepackage[utf8]{inputenc}
\usepackage{amsfonts}
\usepackage{amsmath}
\usepackage{amssymb}
\usepackage{amsthm}
\usepackage{xcolor}
\usepackage{enumitem}  
\usepackage[margin=1.3in]{geometry}
\usepackage{graphicx}
\usepackage[11pt]{moresize}
\usepackage{tikz}
\tikzstyle{vertex}=[circle,draw=black,fill=black,inner sep=0,minimum size=3pt,text=white,font=\footnotesize]
\usepackage{nccmath}
\usepackage{tcolorbox}

\usepackage{hyperref}       % hyperlinks

\newtheorem{theorem}{Theorem}
\newtheorem*{theorem*}{Theorem}
\newtheorem*{conjecture*}{Conjecture}
\newtheorem{proposition}[theorem]{Proposition}%[section]
\newtheorem{lemma}[theorem]{Lemma}

\theoremstyle{remark}
\newtheorem{definition}{Definition}

\newtheorem*{remark*}{Remark}

\newcommand{\sign}{\mathsf{sign}}

\newcommand{\R}{\mathbb{R}}

\newcommand{\N}{\mathbb{N}}

\newcommand{\mc}{\mathcal}

\newcommand{\ep}{\epsilon}

\newcommand{\sub}{\subseteq}

\newcommand{\wh}{\widehat}

\newcommand{\eps}{\varepsilon}

\DeclareMathOperator{\VCdim}{\mathtt{VCdim}}
\DeclareMathOperator{\HSdim}{\mathtt{HS}}

\DeclareMathOperator{\argmax}{argmax}

\DeclareMathOperator{\Ldim}{\mathtt{Ldim}}

\DeclareMathOperator{\List}{\mathtt{LR}}

\title{Replicability and stability in learning}
\author{Zachary Chase}
\thanks{The first author is partially supported by Ben Green's Simons Investigator Grant 376201 and gratefully acknowledges the support of the Simons Foundation.}
\email{zachary.chase@maths.ox.ac.uk}
\address{Mathematical Institute, Andrew Wiles Building, Radcliffe Observatory Quarter, Woodstock Road, Oxford, UK}
\author{Shay Moran}
\thanks{Shay Moran is a Robert J.\ Shillman Fellow; he acknowledges support by ISF grant 1225/20, by BSF grant 2018385, by an Azrieli Faculty Fellowship, by Israel PBC-VATAT, by the Technion Center for Machine Learning and Intelligent Systems (MLIS), and by the the European Union (ERC, GENERALIZATION, 101039692). Views and opinions expressed are however those of the author(s) only and do not necessarily reflect those of the European Union or the European Research Council Executive Agency. Neither the European Union nor the granting authority can be held responsible for them.}
\email{smoran@technion.ac.il}
\address{Faculty of Mathematics, Technion-IIT, Haifa, Israel}
\author{Amir Yehudayoff}
\email{amir.yehudayoff@gmail.com}
\address{Faculty of Mathematics, Technion-IIT, Haifa, Israel}

\date{}%October 3, 2022}

\begin{document}

\begin{abstract}
Replicability is essential in science as it allows us to validate and verify research findings. Impagliazzo, Lei, Pitassi and Sorrell (`22) recently initiated the study of replicability in machine learning. A learning algorithm is replicable if it typically produces the same output when applied on two i.i.d.\ inputs using the same internal randomness. 
We study a variant of replicability that does not involve fixing the randomness.
An algorithm satisfies this form of replicability if it typically produces the same output when applied on two i.i.d. inputs
(without fixing the internal randomness).  
 This variant is called global stability and was introduced by Bun, Livni and Moran ('20) in the context of differential privacy.

Impagliazzo et al.\ showed how to boost any replicable algorithm so that it produces the same output with probability arbitrarily close to 1. In contrast, we demonstrate that for numerous learning tasks, global stability can only be accomplished weakly, where the same output is produced only with probability bounded away from 1. To overcome this limitation, we introduce the concept of list replicability, which is equivalent to global stability. Moreover, we prove that list replicability can be boosted so that it is achieved with probability arbitrarily close to 1. We also describe basic relations between standard learning-theoretic complexity measures and 
list replicable numbers.  Our results, in addition, imply that besides trivial cases, replicable algorithms (in the sense of Impagliazzo et al.) must be randomized. 

The proof of the impossibility result is based on a topological fixed-point theorem. For every algorithm, we are able to locate a ``hard input distribution” by applying the Poincar\'{e}-Miranda theorem in a related topological setting. The equivalence between global stability and list replicability is algorithmic.
\end{abstract}

\maketitle

\section{Introduction}
Replicability is a basic principle underlying the scientific method.
 A study is replicable if it \emph{reliably yields the same results when conducted again 
 using new data}.\footnote{Replicability is closely related to reproducibility: 
 the difference between the two notions is that in reproducibility, one applies the same methods using the {\em same} data, whereas in replicability the data is resampled. Reproducibility is reminiscent of the notion of pseudo-deterministic algorithms~\cite{gat2011probabilistic},
 which are (randomized) algorithms that on each input 
 have a unique output with high probability.} 
 
 Replicability was recently defined and investigated in the context of machine learning~\cite{impagliazzo2022reproducibility,ahn2022reproducibility,esfandiari2022reproducible,bun2023stability}.
%Background and motivation for studying replicability are provided in detail in the introduction of~\cite{impagliazzo2022reproducibility}.
The pioneering work by \cite{impagliazzo2022reproducibility} introduced the notion of replicable learning algorithms:
a learning algorithm is \emph{replicable} if it typically produces the same output when applied on two i.i.d.\ inputs using the same internal randomness. This definition requires using the same internal randomness on two independent executions. 
There are scenarios, however, for which it is difficult to use exactly the same internal randomness; e.g., when the algorithm is executed in distant locations, or when the randomness is not recorded. 

Thus, continuing the investigation of (informally speaking) replicability, in this work we study learning algorithms that \emph{typically produce the same predictor when applied on two i.i.d.\ inputs} (but without sharing the randomness).
This form of replicability is defined in the literature as {\em global stability}: it was introduced by~\cite{bun2020equivalence}
and refined by~\cite{ghazi2021sample,GhaziKM21} as a mean towards designing differentially private learning algorithms. 

In routine experiments, it is most desirable to have replicability that generates the same output $99\%$ of the time, while global stability may be satisfied only ~$1\%$ of the time. Naturally, then, the following question guides our work:
%In practice, it is desirable to have a high degree of replicability. 
    \begin{tcolorbox}
\begin{center}
    Can replicability be boosted?
\end{center}
\noindent {\em Can an algorithm that outputs the same predictor 
$1\%$ of the time be boosted to an algorithm that produces the same predictor $99\%$ of the time?}
\end{tcolorbox}

%In order to emphasise these differences between
%replicability and global stability, 
%we introduce a new notion that we call {\em independent replicability
%(iReplicability)}. 
%Formally speaking, the definitions of global stability 
%and iReplicability are equivalent.
%This terminology also highlights that we care about iReplicability 
%as an end and not just as a means. 

The authors of~\cite{impagliazzo2022reproducibility} showed how to boost replicabile learning algorithms. 
Our first result is a negative answer to the above boosting question with respect to global stability.
We show that global stability cannot, in general, be boosted. 
Specifically, we exhibit concept classes that can be learned 
with global stability parameter $\rho = 1\%$, but cannot be learned with parameter larger than~$1\%$. Consequently, it is not always feasible to obtain desirable levels of global stability in learning tasks.

Our second main contribution circumvents this impossibility result
by introducing the notion of \emph{list replicability}.
We show that an algorithm with global stability parameter~$\rho$ can be effectively converted
into an algorithm that with high probability outputs a predictor from a {\em fixed data-independent list} of size at most $\tfrac{1}{\rho}$. 
Instead of a single predictor that is typically outputted, 
there is a fixed short list that nearly always contains the output predictor.
Notice that this form of replicability can be tested and verified by publishing the short list and testing whether it typically contains the output.

A similar discrepancy occurs in coding theory.
When the amount of noise is small,
the option of unique decoding is available,
but when the amount of noise is high,
the only thing one can aim for is list decoding. 

%In Section\ref{definitions}, we state the precise definitions we use in this paper, review related definitions, for example introduced in ~\cite{impagliazzo2022reproducibility}, and discuss the pertinent distinctions. 
%In Section \ref{a} \red{finish}

\section{Main definitions}\label{definitions}

The purpose of this section is to formally state the main definitions
we introduce, and to compare them with related replicability notions.

Let $X$ be a set and $\mc{H} \sub \{\pm\}^X$ a hypothesis class\footnote{We use $\{\pm\}$ as shorthand for $\{-1,+1\}$.}.
We use standard notation and terminology;
for the basic definitions of PAC learning we refer to the book~\cite{shalev2014understanding}.
A predictor is a function from $X$ to $\{\pm\}$.
A learning rule $\mc{A}$ maps a data sample $S \in (X\times\{\pm\})^n$
to a predictor $\mc{A}(S)$.
We will apply $\mc{A}$ on a sample $S$ that comes from a product distribution $\mc{D}^n$
that is unknown to the algorithm. 
The population loss of $h$ with respect to 
$\mc{D}$ is denoted by $L_\mc{D}(h)$,
and the empirical loss with respect to data $S$ is denoted by $L_S(h)$.
We assume that the ground set {$X$ is countable}
%and the class $\mc{H}\subseteq\{ \pm \}^X$ are countable 
%\textcolor{red}{Do we really want to assume $\mc{H}$ is countable? Does it help us in anyway (improper algorithms, which are sometimes necessary, can still have an uncountable range)? I think it is enough to assume that $X$ is countable }
so that we can ignore measurability issues. 
A distribution $\mc{D}$ is realizable by $\mc{H}$ if
$\inf \{ L_\mc{D}(h) : h \in \mc{H}\} = 0$.

\subsection{Global stability}

As a means to prove that the 
Littlestone dimension captures PAC learning with differential privacy, ~\cite{bun2020equivalence} introduced the notion of global stability.
However, motivated by the study of replicability, 
it is natural to view global stability as an end in and of itself rather than as just a means.

\begin{definition}
A learning rule $\mc{A}$ is called $\rho$-globally stable if for any distribution $\mc{D}$ over inputs, there exists a predictor $h_\mc{D}$ such that
\begin{equation}\label{globally-stable}
\Pr_{S \sim \mc{D}^n}[\mc{A}(S) = h_\mc{D}] \geq \rho.
\end{equation}
%If $\mc{A}$ is randomized, then the probability is taken not only over the draw of $S \sim \mc{D}^n$, but also over the internal randomness of $\mc{A}$. 
\end{definition}

There is a qualitatively equivalent way to state the same definition. 
A learning rule $\mc{A}$ is called $\rho$-globally stable if for any distribution $\mc{D}$ over inputs, it holds that
\begin{equation}\label{global-stability-2}
\Pr_{S,S'\sim \mc{D}^n}[\mc{A}(S) = \mc{A}(S') ] \geq \rho ,
\end{equation}
where $S$ and $S'$ are independent. 
To see the equivalence between the two definitions,
for a given distribution $\mc{D}$, denote by $P(f) :=\Pr_{S \sim \mc{D}^n} [\mc{A}(S) = f]$,
and observe that\footnote{
%We assume that the number of possible outputs is countable.
The set $\{h : P(h)>0\}$ is always countable.}
$$\Pr_{S,S'}[\mc{A}(S) = \mc{A}(S') ]
= \sum_f P(f)^2$$
and
\[ \max_{f} P(f) \geq \sum_f P(f)^2 \geq \max_{f}P(f)^2. \]
%where the first inequality follows because $\rvert\|v\|_1\|v\|_\infty \geq \|v\|_2^2 $ for every vector $v$, and the second inequality follows by Equation~\ref{eq:repone}.
%In fact, one can define strong replicability by replacing Equation~\ref{eq:repone} by Equation~\ref{eq:repinfty}.
%For the converse direction, notice that if $\mc{A}$ satisfies Equation~\ref{eq:repinfty} then it satisfies Equation~\ref{eq:repone} with~$\rho^2$ instead of~$\rho$.
The two definitions are thus equivalent up to a quadratic loss in $\rho$.

Globally stable algorithms are not so interesting unless they also provide accuracy guarantees. 
A learning rule $\mc{A}$ is called $(\rho,\ep)$-globally stable 
for $\mc{H}$ if
there exists $n$ so that for every distribution $\mc{D}$ that is realizable by $\mc{H}$, there exists a predictor $h = h_\mc{D}$ such that $L_{\mc{D}}(h) < \epsilon$ and 
\[
\Pr_{S \sim \mc{D}^n}[\mc{A}(S) = h] \geq \rho.
\]
%The global stability of a hypothesis class, defined next, is the largest stability parameter $\rho$ one can attain if the likely-outputted predictor $f$ (in \eqref{globally-stable}) has small error with respect to $\mc{D}$.  
We say that the class $\mc{H}$ is learnable with global stability parameter $\rho$ 
    if for every $\eps>0$ there exists an $(\eps,\rho)$-globally stable learner for $\mc{H}$.
This leads to one of the main definitions of this text.

\begin{definition}
The global stability parameter $\rho(\mc{H})$ of the class $\mc{H}$ is the supremum over all $\rho \in [0,1]$ for which
$\mc{H}$ is learnable with global stability parameter $\rho$.
  We say that $\mc{H}$ is globally stable learnable if $\rho(\mc{H}) >0$.
\end{definition}

Every finite class $\mc{H}$ is globally stable with $\rho(\mc{H})\geq 1/\lvert\mc{H}\rvert$.
The global stability number is defined for every concept class $\mc{H}$, even if $\mc{H}$ is not globally stable, whence $\rho(\mc{H})=0$.
It is an asymptotic measure of how ``stable'' or ``replicable'' 
learning over $\mc{H}$ can potentially be.
The larger the global stability parameter of~$\mc{H}$ is,
the more ``replicable'' learning over $\mc{H}$ can be.

The definition of globally stable learners is distribution-free in the sense that the sample complexity bound $n$ applies to all realizable distributions. 
% It is worth noting that globally stable learners are not PAC learners in the classical sense.
% They are not required to typically output a predictor with small loss,
% but just to sometimes output the same predictor with small loss. 
It is worth noting the following monotonicity:
if $\mc{H}$ has an $(\eps,\rho)$-globally stable learner with sample size $n$
then it also has an $(\eps',\rho')$-globally stable learner with sample size $n'$
for all $\eps'\geq \eps$, $\rho' \leq \rho$ and $n' \geq n$.

\subsection{Comparison to replicability}

The notion of replicability was recently introduced in ~\cite{impagliazzo2022reproducibility}. 

\begin{definition}
\label{def:replicability}
A learning rule $\mc{A}$ is called $\rho$-replicable if 
for every distribution $\mc{D}$, it holds that
\[
\Pr_{S,S'\sim \mc{D}^n, r} [\mc{A}(S, r) = \mc{A}(S', r) ] \geq \rho ,
\]
where $r$ is the internal randomness of $\mc{A}$,
and $S,S'$ and $r$ are independent. 
\end{definition}

Replicability requires using the same internal randomness
on two independent executions. 
Using exactly the same internal randomness is sometimes difficult or even impossible; e.g., when the algorithm is executed in distant locations, or when the randomness is not recorded. 
Global stability does not include this requirment. 

To further emphasize the distinction between replicabililty and global stability, recall from~\eqref{global-stability-2} that an algorithm is $\rho$-global stability if
\[
\Pr_{S,S'\sim \mc{D}^n, r,r'}[\mc{A}(S,r) = \mc{A}(S',r') ] \geq \rho,
\]
where $S,S',r$ and $r'$ are independent. 
Replicability depends on the specific way the internal randomness is used,
whereas global stability just depends on the input-output behavior
of the algorithm.

Another difference between the two notions is that
globally stable algorithms can be derandomized 
while replicable algorithms can not.
This follows from our results below (see the discussion after Theorem~\ref{thm:boostneg}).
% First, notice that for deterministic algorithms, the two definitions
% coincide. Second, we prove below that there is a class
% with global stability parameter $1\%$.
% Results in~\cite{impagliazzo2022reproducibility} 
% show that this class can be learned with replicablity parameter
% $99\%$.
% The underlying algorithm, therefore, can not be derandomized.

\begin{remark*}
{\em The replicability parameter in~\cite{impagliazzo2022reproducibility}
is of the form $1-\rho$, indicating that it is close to one.
We use $\rho$ instead of $1-\rho$ because (as we prove below)
it is not always possible to have the parameter close to one for global stability.}
\end{remark*}

%\begin{definition}[iReplicable learner]
%A learning rule $\mc{A}$ is an $(\eps,\rho)$-iReplicable learner for %$\mc{H}\subseteq\{ \pm \}^X$ if there exists $n=n(\eps,\rho)$ such that for every distribution $\mc{D}$ that is realizable by $\mc{H}$,
%there exists an hypothesis $h_\mc{D}$ such that
%$$L_\mc{D}(h_\mc{D})\leq \eps \quad \text{and} \quad \Pr_{S\sim \mc{D}^n, r \sim \mc{R}}[\mc{A}(S,r) = h_\mc{D}] \geq \rho.$$\end{definition}

%      where the probability is quantified over the input sample $S\sim \mc{D}^n$
%and over the randomness of $\mc{A}$ when $\mc{A}$ is randomized.
%
%\footnote{To simplify presentation, here and below we omit explicit quantification over the internal randomness of $\mc{A}$.}

%It might also be interesting to study the relaxed variant where $n$ is allowed to depend on the target distribution~$\mc{D}$, but this beyond the scope of this work.

\subsection{List replicability}
Instead of allowing a single data-independent prediction~$h$,
we can allow a short list $h_1,\ldots,h_L$ of data-independent predictions.
Global stability says that $h$ is the output of the algorithm with {\em non-negligible probability}.
By allowing lists, we can require that one of $h_1,\ldots,h_L$
is the output {\em with high probability}. 

\begin{definition}[List replicable learning]
A learning rule $\mc{A}$ is called $(\eps,L)$-list replicable learner for the class $\mc{H}$
if for every $\delta>0$, there exists $n=n(\eps,L,\delta)$ such that for every distribution $\mc{D}$ that is realizable by~$\mc{H}$,
there exist hypotheses $h_1,h_2,\ldots,  h_L$ such that
$$\Pr_{S\sim \mc{D}^n}[\mc{A}(S) \in\{h_1,\ldots, h_L\}] \geq 1-\delta$$
and 
for all $\ell \in [L]$,
$$L_\mc{D}(h_\ell)\leq \eps.$$
\end{definition}

Thus, list replicability is stronger than global stability on two fronts: first, we are guaranteed a short list of possible outputs, which of course immediately implies global stability, but, secondly and more importantly, the output \emph{nearly always} belongs to the short list, rather than just non-negligibly often. 

 We say that $\mc{H}$ is learnable with replicability list size $L$
    if for every $\eps>0$, there exists an $(\eps,L)$-list replicable learner for $\mc{H}$.
 
\begin{definition}[List replicability number]
    The list replicability number of $\mc{H}$ is defined as
    \[\List(\mc{H}) := \min\bigl\{L \in \N : \text{$\mc{H}$ is learnable with replicability list size $L$} \bigr\}.\]
    We say that $\mc{H}$ is replicable list learnable if $\List(\mc{H}) < \infty$.
\end{definition}

Every finite class $\mc{H}$ is list replicable learnable with $\List(\mc{H})\leq \lvert\mc{H}\rvert$.
The list replicability number is 
a different asymptotic measure of how ``stable'' learning over $\mc{H}$ can potentially be.
The smaller the list replicability parameter of~$\mc{H}$ is,
the more ``stable'' learning over $\mc{H}$ can be.

\section{Main results}

\subsection{Basic properties}

The first question we address is {\em can global stability be boosted?}
Boosting the global stability parameter arbitrarily close to one
means that we can make algorithms more and more stable. 
This, however, turns out to be impossible in general.

\begin{theorem}\label{thm:boostneg}
For every $\rho_0 >0$, there exists a globally stable class $\mc{H}$ with $\rho(\mc{H}) < \rho_0$.
\end{theorem}

To prove this impossibility result, for a given $\rho_0$ we find a finite class $\mc{H}$
so that for every learner~$\mc{A}$ for~$\mc{H}$ we 
can locate a distribution $\mc{D}$ that exposes the instability of $\mc{A}$.
We develop a mechanism for locating $\mc{D}$ given the structure of $\mc{H}$
and the  ``functionality'' of the algorithm $\mc{A}$. 
The mechanism, interestingly, 
relies on a topological fixed-point theorem.
For more details, see Sections~\ref{sec:geo} and~\ref{sec:instab}.

Theorem~\ref{thm:boostneg} implies that in some cases replicable algorithms must be randomized.
First, notice that for deterministic algorithms, the definitions of replicability and global-stability coincide. 
Second, the theorem says that there is a finite class with global stability parameter $<1\%$.
Results in~\cite{impagliazzo2022reproducibility} show that this class can be learned with replicablity parameter
$99\%$. The underlying algorithm, therefore, cannot be derandomized.

Theorem~\ref{thm:boostneg} implies that global stability numbers can be arbitrarily close to zero.
{\em Can they take any possible value between zero and one?}
It turns out that global stability numbers form a discrete set;
they are always reciprocals of integers (or are zero).

\begin{theorem}
\label{thm:recip}
For every class $\mc{H}$, one has
$$\rho(\mc{H})  = \frac{1}{\List(\mc{H})}.$$
\end{theorem}
%
%
%\begin{theorem}[Boosting: Positive Result]\label{thm:boostpos}
%Let $\mc{H}$ and let $\rho>0$. If $\mc{H}$ is learnable with replicability number $\rho>0$ then $\rho(\mc{H})\geq \lfloor\rho^{-1}\rfloor^{-1} \geq \rho$.
%In particular, $\rho(\mc{H})^{-1}$ is an integer for every strongly replicable class $\mc{H}$.
%
%% $\rho'\geq \rho$ satisfying
%% \[\rho' = \frac{1}{\lfloor 1/\rho\rfloor}.\]
%\end{theorem}
%
%% \textcolor{red}{Shay: there is another type of boosting we should add: assume $\mc{H}$ is a class such that there exists a learning rule $A$, $\gamma>0$, and $n_0\in\mathbb{N}$ such that for every realizable distribution $D$ there is an hypothesis $h=h(D)$ whose loss is at most $\frac{1}{2}-\gamma$ and is outputted by $A$ with probability at least $\gamma$, given an input sample of iid examples from $D$. Then, it follows that $\mc{H}$ is strongly-replicable. The only proof I know of this statement is non-construcitve and indirect (it uses the equivalence with differentially private PAC learning). Would be nice to have a constructive proof or at least a direct one. }
%
%
%
%
%\begin{theorem}[Strong Replicability $=$ List Replicability]\label{thm:listrep}

Theorem~\ref{thm:recip} also holds in the limit case that
$\rho(\mc{H}) = 0$ and $\List(\mc{H}) = \infty$.
It implies a weak form of boosting;
for example, 
if we know that $\rho(\mc{H}) \geq 0.31$,
then we can automatically deduce that $\rho(\mc{H}) \geq \tfrac{1}{3}$.
It also says that we can freely replace $\rho(\mc{H})$
with $\frac{1}{\List(\mc{H})}$ and vice versa
(as we do for the rest of this text).
The theorem, in addition, implies the following equivalence:
$$\text{$\mc{H}$ is globally stable
$\ \iff \ $      $\mc{H}$ is list replicable.}$$
It further shows that although global stability can not be boosted,
we can always find a short list of data-independent predictors
that can be outputted with high probability, thereby providing desirable ``replicability" guarantees. 

\subsection{Other dimensions}

%Furthermore, the optimal list size $\List(\mc{H})$ and replicability number $\rho(\mc{H})$ satisfy
%\[\List(\mc{H})= \rho(\mc{H})^{-1}. \]
%\end{theorem}

The next topic we investigate is relations between list replicability
numbers and other known learning-theoretic measures.
The first measure we address is the VC dimension,
which is of fundamental importance in learning theory~\cite{shalev2014understanding}.
The following theorem shows that large VC dimension
implies instability.

\begin{theorem}
For every $\mc{H}$, we have
$\List(\mc{H}) \geq \VCdim(\mc{H})$.
\end{theorem}

This follows from the following sharp result. 

\begin{lemma}
\label{lem:VC}
For every $d \in \N$, we have $\List(\{\pm\}^d) = d$. 
%\textcolor{red}{Shay: this isn't true for $d=1$ whence $\List(\{\pm 1\}^1)=1$. 
%I think that in general we have $\List(\{\pm\}^d) = d$ (this was also stated in a previous version). }
\end{lemma}

The lemma, together with Theorem~\ref{thm:recip}, immediately implies Theorem~\ref{thm:boostneg} above.
The proof of the lemma consists of two parts.
One part is the lower bound $\List(\{\pm\}^d) \geq d$,
which is proved via the topological mechanism mentioned above.
The second part is the upper bound $\List(\{\pm\}^d) \leq d$,
which is an algorithmic result. 
The main difficulty in the algorithm is guaranteeing stability:
{\em we potentially have $2^d$ functions,
so how can we make sure that almost certainly
only $d$ functions are being outputted?}
The proof appears in Section~\ref{sec:VCbounds}.

A second important dimension in the theory of learning
is the Littlestone dimension, which is known to be deeply
linked to {\em privacy} and to {\em stability}.
The authors of~\cite{alon2022private} showed that the Littlestone dimension
provides a lower bound on the sample complexity 
of differentially private PAC learning.
The authors of~\cite{bun2020equivalence} complemented the picture by showing
that the Littlestone dimension also implies an upper bound
on the sample complexity of differentially private PAC learning.
The two  main results of the latter work state that
(i) finite Littlestone dimension yields globally stable algorithms, and
(ii) globally stable algorithms lead to differentially private PAC learning algorithms.
Altogether, the results in these two papers lead to the following equivalence:
\begin{equation}
\Ldim(\mc{H}) < \infty \quad \iff \quad \List(\mc{H}) < \infty \label{eqn:little}
\end{equation}
that holds for every class $\mc{H}$.
A concrete bound that was proved in~\cite{bun2020equivalence} says that for every class~$\mc{H}$,
$$\List(\mc{H})\leq 2^{2^{O(\Ldim(\mc{H}))}}.$$
The more recent work~\cite{ghazi2021sample} developed more efficient differentially private algorithms.
Ideas from this work can be used to improve the bound:
for every $\mc{H}$, it holds that
$$\List(\mc{H})\leq 2^{O(\Ldim(\mc{H})^2)}.$$
%\textcolor{red}{Should be $2^{O(\Ldim(\mc{H})^2)}$?}
We do not know if this last bound is sharp. But we do prove that
a bound in the other direction does not hold.

\begin{theorem}
\label{thm:Little}
For every integer $d \geq  2$, there is a finite class $\mc{H}$ with $\Ldim(\mc{H})=d$ and $\List(\mc{H})=2$.
\end{theorem}

The class in the theorem above can be taken to be
the collection of thresholds over $2^d$ points.
It was not a priori clear to us that the list size of this class is two.
The algorithm and its analysis rely on the special 
%topological \textcolor{red}{in what sense are they topological? I think combinatorial is better and more accessible to most potential readers.} 
properties of {a linear order}.
We do not know if a similar bound holds, for example,
for halfspaces (LTFs) in the plane.

A third quantity we compare list replicability numbers to is {\em hollow star} numbers.\footnote{The hollow star number $\HSdim(\mc{H})$ of $\mc{H}$ is the maximum size of a subset
$X'$ of $X$ so that the projection $\mc{H}|_{X'}$ of $\mc{H}$ to $X'$
contains a hollow star
(that is, there is a function $f \not \in \mc{H}|_{X'}$
so that all the $|X'|$ functions of Hamming distance one from $f$
are in $\mc{H}|_{X'}$).} 
Hollow star numbers are known to be related to proper learning~\cite{bousquet2020proper};
namely, to learning algorithm whose output is restricted to be in $\mc{H}$.
We define $\List_p(\mc{H})$ in the same way as $\List(\mc{H})$
except that we quantify only over proper algorithms.

\begin{theorem}
For every class $\mc{H}$, one has
$$\List_p(\mc{H}) \geq \HSdim(\mc{H})-1.$$
\end{theorem}

The theorem is sharp in the sense that for every integer $s>1$, 
the class of~$s$ singletons has hollow star number $s$
and its proper list replicability number is exactly~$s-1$.
This implies a separation between proper and non-proper list numbers, because the (non-proper) list replicability number
of the class of singletons is two.
The proof of the theorem uses, again, the topological mechanism
mentioned above; see Section~\ref{sec:HS}.

\subsection{A geometric picture}
\label{sec:geo}
There is a geometric way to picture global stability.
Let~$\mc{H}$ be a globally stable class.
Denote by $\Delta = \Delta_\mc{H}$ the collection of $\mc{H}$-realizable
distributions. We think of $\Delta$ as a metric space equipped with the total-variation distance (denoted by $\mathtt{TV}$).
We find it helpful to think of $\Delta$ geometrically, as a subset of the space $\R^{X \times \{\pm\}}$ whose points
are distributions that are realizable by $\mc{H}$.

We argue below that global stability means that we can color the space 
$\Delta$ in a suitable manner so that locally there are only few colors. 
This perspective might be useful for understanding the topological mechanism for proving instability that is described in Section~\ref{sec:instab}.

Let $L=\List(\mc{H})$, and 
let $\mc{A}$ be a learning rule witnessing it
for some $\epsilon >0$.
We use~$\mc{A}$ to color $\Delta$ as follows.
Pick $n$ so that
\[\Pr_{S\sim \mc{D}^n}[\mc{A}(S)\in \mc{L}] \geq 1- \frac{1}{3L}, \]
where $\mc{L}$ is the list guaranteed by list replicability (so that $\lvert \mc{L}\rvert = L$ and $L_\mc{D}(h)\leq \eps$ for all $h\in\mc{L}$).
The color of a distribution $\mc{D}\in \Delta$ is a hypothesis %$h=h_{\mc{D}}$
that is most frequently outputted by $\mc{A}$ when applied on $S\sim \mc{D}^n$. That is,
\begin{align*}
h_{\mc{D}} = \mathop{\argmax}_h \Pr_{S\sim \mc{D}^n}[\mc{A}(S) = h], 
\end{align*}
where ties are broken arbitrarily.
The coloring $\mc{D}\mapsto h_\mc{D}$ satisfies the following:
\begin{enumerate}
\item $L_\mc{D}(h_{\mc{D}})\leq \epsilon$ for every $\mc{D}\in \Delta$.
\item Every $\mc{D}\in \Delta$ has a small neighborhood with at most $L$ colors:
for $\delta = \frac{1}{3nL}$, we have
for all $\mc{D} \in \Delta$, that
\[\bigl\lvert\bigl\{h_{\mc{D}'} : \mathtt{TV}(\mc{D},\mc{D}')\leq \delta \bigr\}\bigr\rvert\leq L.\]
\end{enumerate}
Item~1 holds by the choice of $\mc{A}$. 
Item~2 holds because if
$\mathtt{TV}(\mc{D}, \mc{D}') \le \delta$, then
\[\mathtt{TV}(\mc{D}^n, (\mc{D}')^n) < 
\frac{1}{3L},\] which implies that
$\Pr_{S\sim \mc{D'}^n}[A(S) \notin \mc{L}] <\frac{2}{3L}$,
and so $h_{\mc{D}'}\in \mc{L}$.
%Thus, Item~2 holds with $\delta = \frac{1}{3nL}$.
%and $\mc{L}$ being the set of colors in the $\delta$-neighborhood of $\mc{D}$.

We finish this section by remarking that for finite classes there is a correspondence between list replicable learners and colorings of $\Delta$.
\begin{theorem}
The following statements are equivalent for a finite class $\mc{H}$:
\begin{enumerate}
\item $\List(\mc{H})\leq L$.
\item For every $\eps>0$, there exist $\delta>0$ and a coloring $\mc{D}\mapsto h_\mc{D}$ of $\Delta(\mc{H})$
such that~$L_\mc{D}(h_{\mc{D}})\leq \epsilon$ and
$\lvert\{h_{\mc{D}'} : \mathtt{TV}(\mc{D},\mc{D}')\leq \delta \}\rvert\leq L$ for every $\mc{D}\in \Delta$.
\end{enumerate}
\end{theorem}
The implication $1\implies 2$ was explained above. 
For the direction $2\implies 1$, because $\mc{H}$ is finite,
for a sufficiently large sample size $n$, the empirical distribution $\hat{\mc{D}}$ induced by a sample $S\sim \mc{D}^n$ is w.h.p.\ close to $\mc{D}$ in total-variation. Thus, the learning rule that outputs the hypothesis $h=h_{\hat{\mc{D}}}$ is a list replicable learner for $\mc{H}$.

\subsection{Future directions}
There are many questions that can be addressed in future works. 
One question is about agnostic learning.
Let $\mc{H}$ be a globally stable class. 
Roughly speaking, this means that there is an algorithm that is stable under 
{\em realizable} distributions.
{\em Is $\mc{H}$ globally stable in the agnostic setting?
I.e., roughly speaking, is it true that $\mc{H}$ is stable under general distributions?}

A second direction of research is obtaining stronger algorithmic results.
We highlight specific algorithmic questions for further study.
Let $\mc{H}$ be a finite class of size~$m$.
It is trivially true that $\List(\mc{H}) \leq m$.
{\em How much can we improve on this trivial bound?
Is it true, e.g., that $\List(\mc{H})\leq C\log m$?}
In fact,
{\em is it true that $\List(\mc{H})\leq  \VCdim(\mc{H}) + 1$?}
The answers to these questions can potentially lead to new ways
for producing stable algorithms.

We mention that for infinite classes the answer to the last question is no. 
Take, for example, a class $\mc{H}$ with an infinite Littlestone dimension and a finite VC dimension.
The statement in~\eqref{eqn:little} implies that $\List(\mc{H}) = \infty$,
although $\VCdim(\mc{H}) < \infty$. 

A different direction for future study is adding to the global stability requirement
a metric ingredient.
Assume that there is some measure of distance $\mathtt{dist}$ between predictors in $\{\pm 1\}^X$. Instead of asking that typically $\mc{A}(S) = \mc{A}(S')$,
we can ask that typically the distance $\mathtt{dist}(\mc{A}(S),\mc{A}(S'))$ is small.
This corresponds to scenarios where we do not want full stability,
but only approximate stability.

%In the next section we introduce a parameter which agrees with $\List(\mc{H})$ for finite $\mc{H}$ 
%and extends the last question to infinite classes. 

%\subsubsection{Coloring Distributions}
%Let $\mc{H}\subseteq \{-,+\}^n$, and let $\Delta=\Delta(\mc{H})$ denote the space of all distributions that are realizable by $\mc{H}$. Notice that every~$p\in \Delta$ can be represented by a $2n$-dimensional probability vector over $[n]\times\{-,+\}$.
%For distributions $p,q\in \Delta$, let $p\ll q$ denote that $p$ is absolutely continuous with respect to $p$, i.e.\ that $\mathtt{support}(p)\subseteq \mathtt{support}(q)$. 
% We say that $N\subseteq \Delta$ is a \emph{local neighborhood} of $p\in \Delta$
% if the following holds:
% \begin{itemize}
% \item $p\in N$,
% \item $N=E\cap\{q : q\ll p\}$, where $N\subseteq \Delta$ is an open set.
% \end{itemize}

\section{Instability mechanism}
\label{sec:instab}

In this section we develop a mechanism for proving
instability results; that is, for bounding from above
the global stability parameter of a concept class,
or bounding from below the list replicability number.

\subsection{Empirical learners}

The first step is noticing that globally stable learners 
have low empirical error (without loss of generality). 
We say that algorithm $\mc{A}$ is an $\eps$-empirical learner for $\mc{H}$ 
if there is $n_0$ so that  %for every $\delta>0$ 
for all $n>n_0$ and for all $S \in (X\times\{\pm\})^n$ that is consistent with $\mc{H}$
and randomness $r$,
we have
$L_S(\mc{A}(S,r)) \leq \eps$.
 
 \begin{lemma}\label{empiricalwlog}
If there is an $(\eps,\rho)$-globally stable learner for $\mc{H}$
then for every $\delta > 0$, there is an
$(\eps,\rho-\delta)$-globally stable learner for $\mc{H}$
that is also an $(\eps+\delta)$-empirical learner. 
 \end{lemma}

 \begin{proof}
Let $\mc{A}$ be an $(\eps,\rho)$-globally stable learner for $\mc{H}$.
Define a new algorithm $\mc{A}'$ as follows.
For every $S$ that is consistent with $\mc{H}$ and randomness $r$,
if $L_S(\mc{A}(S,r)) \leq \eps+\delta$ then define $\mc{A}'(S,r) = \mc{A}(S,r)$,
and if $L_S(\mc{A}(S,r)) > \eps+\delta$ then define $\mc{A}'(S,r)$ 
to be any hypothesis with $L_S(h) < \eps+ \delta$.

We claim that $\mc{A}'$ is an $(\eps,\rho - \delta)$-globally stable learner for $\mc{H}$.
Let $n$ be large enough so that 
for every distribution $\mc{D}$ that is realizable by $\mc{H}$,
there exists an hypothesis $h_\mc{D}$ so that
$L_\mc{D}(h_\mc{D})\leq \eps$, so that
$\Pr_{S,r}[A(S,r) = h_\mc{D}] \geq \rho$
and so that $\Pr[|L_S(h_\mc{D})-L_\mc{D}(h_\mc{D})| \geq \delta] < \delta$.
It follows that
\begin{equation*}
\Pr[\mc{A}'(S,r) = h_\mc{D}]
%\geq \Pr[\mc{A}(S,r) = h_\mc{D}] - \Pr[|L_S(h_\mc{D})-L_\mc{D}(h_\mc{D})|>\delta]
\geq \rho - \delta . \qedhere
\end{equation*}
\end{proof}

\subsection{A mechanism}

{\em What are the properties make a concept class $\mc{H}$ unstable in the sense that
$\rho(\mc{H})$ is small?}
We identify the following witness of difficulty. 
%The witness immediately allows to upper bound
%the global stability parameter
%$\rho(\mc{H})$, which is defined by quantifying only over proper learners (see Proposition~\ref{instabilitywitnesspf} below).
%Later on, we observe that it can also allow
%to bound the (non-proper) parameter~$\rho(\mc{H})$.

 The witness applies to the following version of $\rho(\mc{H})$.
  Let $\mc{F}$ be a class. Think of $\mc{F}$ as the possible outputs of the learner.
 In the proper setting, $\mc{F} = \mc{H}$, whereas in the fully improper case $\mc{F}=\{\pm\}^X$.
The global stability parameter
of $\mc{H}$ with respect to $\mc{F}$, denoted by $\rho_\mc{F}(\mc{H})$, is defined by quantifying over learners 
with output in $\mc{F}$. 
 
%An instability witness of size $k$ for $\mc{H}$ with respect to $\mc{F}$  is a function $W : [k] \times \{\pm\} \to X \times \{\pm\}$ that satisfies property~\eqref{it:real} 
%from Definition~\ref{instabilitywitness}
%and instead of~\eqref{it:hard} the following condition holds:
%There are $k+1$ subsets $\mc{F}_0,\mc{F}_1,\ldots,\mc{F}_k$ 
%that form a partition of $\mc{F}$ so that the following hold:

% \begin{itemize}
% \item[(a)]  For every $j \in [k]$ and $f_0 \in \mc{F}_0$ we have $f_0(W_X(j,-)) \not = W_\pm(j,-)$.

% \item[(b)] For every $j \in [k]$ and $f_j \in \mc{F}_j$ we have $f_j(W_X(j,+)) \neq W_\pm(j,+)$.
 %\end{itemize}
 %}
 %\end{remark*}

 \begin{definition}[Instability witness]\label{instabilitywitness} 
 An instability witness of size $k$ for $\mc{H}$
 with respect to $\mc{F}$
 is a function $W : [k] \times \{\pm\} \to X \times \{\pm\}$.
 The witness $W$ is composed of two functions
 $W_X : [k] \times \{\pm\} \to X$
 and $W_\pm : [k] \times \{\pm\} \to \{\pm\}$.
It satisfies the following requirements:

% 
% pair of ordered $k$ tuples 
% $$\big((x_1,\sigma_1),\dots,(x_k,\sigma_k)\big) \quad \text{and} \quad
%\big((x'_1,\sigma'_1),\dots,(x'_k,\sigma'_k)\big)$$ 
% where $x_i,x'_i \in X$ and $\sigma_i,\sigma'_i \in \{\pm\}$ for all $i \in [k]$
%so that the following hold:
 \begin{enumerate}
 \item There is $(x,y) \in X \times \{\pm\}$ so that for every $\sigma \in \{\pm 1\}^k$,
 there exists $h \in \mc{H}$ such that $h(x) = y$ and for all $j \in [k]$, it holds that
$W(j,\sigma_j)\neq (x,y)$ and $h(W_X(j,\sigma_j)) = W_\pm(j,\sigma_j)$.
% \item For every $\sigma \in \{\pm 1\}^k$,
%  there is $(x,y)$ that is not
%  in the image of $j \mapsto W(j,\sigma_j)$ and
% there is $h \in \mc{H}$ so that 
% $h(x) = y$ and $h(W_X(j,\sigma_j)) = W_\pm(j,\sigma_j)$ for all $j \in [k]$.
\label{it:real}

 \item There are $k+1$ subsets $\mc{F}_0,\mc{F}_1,\ldots,\mc{F}_k$ 
that form a partition of $\mc{F}$ so that the following hold:
\label{it:hard}

 \begin{itemize}
 \item[(a)]  For every $j \in [k]$ and $f_0 \in \mc{F}_0$ we have $f_0(W_X(j,-)) \not = W_\pm(j,-)$.

 \item[(b)] For every $j \in [k]$ and $f_j \in \mc{F}_j$ we have $f_j(W_X(j,+)) \neq W_\pm(j,+)$.
 \end{itemize}
 
% \item {The image of $W$ is not all of $X \times \{\pm\}$.}

 \end{enumerate}
  \end{definition}

 \begin{proposition}\label{instabilitywitnesspf}
If there exists an instability witness of size $k$ for a class $\mc{H}$ with respect to a class $\mc{F}$, then 
$$\rho_\mc{F}(\mc{H}) \le \tfrac{1}{k+1}.$$
 \end{proposition}

The proof of the proposition is based on a topological argument.
Specifically, we use of the following theorem, 
conjectured by Poincar\'{e} in 1883 and proved by Miranda in 1940. 
Miranda in fact showed that the theorem is equivalent to Brouwer's fixed-point theorem and to Sperner's lemma.

 \begin{theorem}\label{poincaremiranda}
Let $g_1,\dots,g_n : [-1,1]^k \to [-1,1]$ be $k$ continuous functions. 
Suppose that for every $t = (t_1,\ldots,t_n) \in [-1,1]^k$ and each $j \in [k]$, 
if $t_j = -1$ then $g_j(t) \le 0$ and 
if $t_j = 1$ then $g_j(t) \ge 0$. Then, 
there exists $t^* \in [-1,1]^k$ so that $g_j(t^*) = 0$ for all $j \in [k]$.
 \end{theorem}

 % Although we shall use Theorem \ref{poincaremiranda} in the argument below, with slight modifications we could instead use Sperner's lemma directly. \red{maybe add the Sperner proof for 2 shattered points and a dummy?}

 \begin{proof}[Proof of Proposition~\ref{instabilitywitnesspf}]
 Assume that there exists an instability witness of size $k$ for~$\mc{H}$.
 Let $\eps,\delta > 0$ be small.
 By Lemma~\ref{empiricalwlog}, 
 we can assume that $\mc{A}$
is an $(\eps,\rho-\delta)$-globally stable learner for $\mc{H}$
that is also an $\eps$-empirical learner. 

For $t \in [-1,1]^k$, define a distribution $\mc{D}_{t}$
on $X\times\{\pm\}$ as follows.
The mass $\mc{D}_t$ gives to $(x,\sigma)$ is
$$\mc{D}_t\big( (x,\sigma) \big) =  \sum_j\tfrac{|t_j|}{k}$$ 
where the sum is over all $j \in [k]$ so that\footnote{The sign function 
$\sign$ is $1$ on $[0,\infty)$
and $-1$ on $(-\infty,0)$.}
$$W\big((j,\sign(t_j)) \big) = (x,\sigma).$$
The rest of the mass $1- \frac{|t_1|+\ldots+|t_k|}{k}$ is placed on
the point $(x,y)$ from property~\eqref{it:real} of $W$.
It follows that $\mc{D}_t$ is realizable,
and that the map $t \mapsto \mc{D}_t$
is continuous (with respect to, say,
total variation distance).

%1-|t_1|-\dots-|t_k| & (x_,+)\end{cases},$$ where, by a slight abuse of notation, the pair $(z,\omega)$ gets weight $|t_{j_1}|+\dots+|t_{j_r}|$ if all $(z_{j_1},\omega_{j_1}),\dots,(z_{j_r},\omega_{j_r}) = (z,\omega)$. 

Define functions $g_1,\dots,g_k : [-1,1]^k \to [-1,1]$
by
$$g_j(t) := \Pr \left[\mc{A}(S) \in \mc{F}_0\right]  - \Pr \left[\mc{A}(S) \in \mc{F}_j\right] + \delta t_j,$$
where the probabilities are over 
$S \sim \mc{D}_{t}^n$ and the internal randomness of $\mc{A}$.
The functions $g_1,\dots,g_k$ are continuous on $[-1,1]^k$. 

We need to prove that $g_1,\ldots,g_k$ 
satisfy the assumptions of the Poincar\'{e}-Miranda theorem. 
First, fix $t,j$ so that $t_j = -1$.
There is a finite subset of $X \times \{\pm\}$
so that all distribution $\mc{D}_t$ are supported on it. It follows that if $n$ is large, then 
$$\Pr_{S \sim \mc{D}_t^n} \big[ \exists f \in \mc{F}_0 \ |L_\mc{D}(f) - L_S(f)| > \delta \big] < \delta.$$
By the definition of instability witness
every hypothesis $f_0 \in \mc{F}_0$ satisfies
$$f_0(W_X(j,-)) \not = W_\pm(j,-).$$
Because $\mc{D}_t(W(j,-)) \geq \tfrac{1}{k}$,
it follows that $L_{\mc{D}_t}(f_0)  \geq \tfrac{1}{k}$.
Because $\mc{A}$ is an empirical learner
and $\tfrac{1}{k} > \eps+\delta$, we have
$$\Pr \left[\mc{A}(S) \in \mc{F}_0\right] 
\leq \Pr \big[ \exists f \in \mc{F}_0 \ L_S(f) < \eps  \big] < \delta.$$
Second, for $t,j$ so that $t_j = 1$, we similarly have
$$\Pr\left[\mc{A}(S) \in \mc{F}_j\right] < \delta .$$
It indeed follows that the $k$ functions 
satisfy the assumptions of the Poincar\'{e}-Miranda theorem. 

We thus deduce that there exists $t^* \in [-1,1]^k$ so that for all $j \in [k]$ we have $g_j(t^*) = 0$. 
By the definition of instability witnesses, $\mc{D}_{t^*}$ is realizable by $\mc{H}$. 
Under $\mc{D}_{t^*}$, we have for all $j \in [k]$ that
$$\Big|\Pr \left[A(S) \in \mc{F}_0\right]  - \Pr \left[A(S) \in \mc{F}_j\right]\Big| \le \delta .$$
Because $\mc{F}_0,\ldots,\mc{F}_k$ are disjoint,
$$\Pr \left[A(S) \in \mc{F}_0\right]+\sum_{j} \Pr \left[A(S) \in \mc{F}_j\right] \leq 1.$$  Hence, for all $j \in \{0,1,\ldots,k\}$,
$$\Pr \left[A(S) \in \mc{F}_j \right] \le \tfrac{1}{k+1}+k\delta.$$ 
%and $$\Pr_{S \sim \mc{D}_{t^*}^n} \left[A(S) \in \mc{H}_j\right] \le \frac{1}{k+1}+k\delta$$ for each $j \in \{1,\dots,k\}$. 
Because the $\mc{F}_j$'s partition $\mc{F}$, 
we can deduce $\Pr \left[A(S) = f\right] \le \frac{1}{k+1}+k\delta$ for every $f \in \mc{F}$. 
The proof is complete, since we can take $\delta$ to be as small as we wish.
 \end{proof}

\section{List numbers}

In this section we prove that
$\rho(\mc{H}) \cdot \List(\mc{H}) = 1$.
This equality is comprised of two inequalities
$\rho(\mc{H}) \cdot \List(\mc{H}) \geq 1$
and $\rho(\mc{H}) \cdot \List(\mc{H}) \leq 1$.
The former inequality is simple to prove:
an $(\eps,L)$-list replicable learner is an $(\eps, \tfrac{1-\delta}{L})$-globally stable learner for every $\delta>0$. The other inequality follows from the next proposition.
\begin{proposition}\label{prop:reptolist}
Assume $\mc{A}$ is a $(\eps,\rho)$-globally stable learner for a class $\mc{H}$.
Then, there is a $(2\eps, L)$-list replicable learner $\mc{A}'$ for $\mc{H}$
with $L \leq \tfrac{1}{\rho}$. 
\end{proposition}

The algorithm $\mc{A}'$ can be efficiently implemented given only oracle access to $\mc{A}$.

\begin{proof}[Proof of Proposition~\ref{prop:reptolist}.]
Let $n_0=n_0(\eps,\rho)$ denote the sample size of $\mc{A}$.
Define the following learning rule $\mc{A}'$.
Let $\delta>0$ be a confidence parameter.
Let $L =\lfloor \tfrac{1}{\rho} \rfloor$ so that $\rho\in (\frac{1}{L+1},\frac{1}{L}]$.
Let $$\alpha = \rho - \tfrac{1}{L+1}>0,$$
let $T = O(\frac{\log(1/\delta)}{\alpha^2})$ be an integer, let $n_1 =  T \cdot n_0$
and let $n_2 = O(\frac{\log(L/\delta)}{\eps^2})$ be an integer,
where $O(\cdot)$ hide universal constants.

\begin{tcolorbox}
\begin{center}
List learner $\mc{A}'$ from learner $\mc{A}$
\end{center}

\ \ \ \
{\bf Input:} A sample $S$ of size $n_1+n_2$.

\ \ \ \ 
 Let $P$ denote the prefix of $S$ consisting of the first $n_1$ examples.

 \ \ \ \ 
Let $Q$ denote the suffix of $S$ consisting of the remaining $n_2$ examples.

\medskip

\begin{enumerate}
\item Partition the examples in $P$ to $T$ batches, each of size $n_0$.
\item Apply $\mc{A}$ on each of these batches.
\item Output an hypothesis $h$ such that
\begin{itemize}
\item[(a)] $h$ was outputted by $\mc{A}$ on at least $(\rho - \tfrac{\alpha}{2})T$ of the batches, and 
\item[(b)] $L_Q(h) \leq \tfrac{3\eps}{2}$.
\end{itemize}
If no such hypothesis exists then output some a priori fixed hypothesis.

\end{enumerate}
\end{tcolorbox}

We now show that $\mc{A}'$ is a $(2\eps,L)$-list replicable learner.
Let $\mc{D}$ be a realizable distribution
and feed a sample $S \sim \mc{D}^{n_1+n_2}$ into $\mc{A}'$.

Denote by $P(h)$ the probability 
$P(h) = \Pr[\mc{A}(S') = h]$, where $S' \sim \mc{D}^{n_0}$.
Denote by $\hat{P}(h)$ the empirical version of $P(h)$
defined as the fraction of times out of the $T$ executions of 
$\mc{A}$ in which $h$ was outputted.
The VC uniform law of large numbers (see e.g.~\cite{shalev2014understanding})
applied to the family of singletons $\{\{h\} : h\in \{\pm\}^X\}$ implies
that $$\Pr \big[ \forall h \in \{\pm\}^X \ |\hat{P}(h) - P(h)| < \tfrac{\alpha}{2} \big]
\geq 1-\tfrac{\delta}{2}.$$
Let $\mc{H}'$ be the (random) set of functions 
$h$ so that $\hat{P}(h) \geq \rho - \tfrac{\alpha}{2}$.
It follows that
\begin{equation}\label{eq:list}
\Pr[ \forall h \in \mc{H}' \ P(h) > \tfrac{1}{L+1} ] \geq 1-\tfrac{\delta}{2}.
\end{equation}
There are at most $L$ hypotheses $h$ satisfying $P(h) > \tfrac{1}{L+1}$.
So, with probability at least $1-\tfrac{\delta}{2}$ every hypothesis $h$ that satisfies
$P(h) > \tfrac{1}{L+1}$ also satisfies
\begin{equation}\label{eq:unif}
\lvert L_\mc{D}(h)  - L_{Q}(h) \rvert < \tfrac{\eps}{2}.
\end{equation}
By the union bound, with probability at least $1-\delta$,
algorithm $\mc{A}'$ outputs an hypothesis $h$ such that $P(h) > \frac{1}{L+1}$ and $L_\mc{D}(h) \leq 2\eps$. 
The list of $h$'s satisfying both inequalities is non-empty,
because $\mc{A}$ is globally stable.
The size of this list is at most $L$.

% Take $\al > 0$ very small based on $\rho$. Then at most $\lfloor \rho^{-1} \rfloor$ hypotheses have probability $\ge \rho-\alpha$ of being outputted by $\mathcal{A}$. Letting $f_1$ achieve the maximum in the definition of $\eta$, we have by uniform convergence \red{issue here. works for proper learners though}
% \textcolor{blue}{Shay: Why? Uniform convergence here is being applied to the family $\{\{h\} : h\in 2^X\}$, which has VC dimension 1.}
% that, with probability at least $1-o_\eta(C)$, no $f \in \{-,+\}^X$ of $\mc{D}$-error at most $\eta-\alpha$ will be outputted by $\wt{A}^C$. This gives the result. The final statement of the lemma follows from restricting the just-given proof to functions of small $\mc{D}$-error (maybe have to change definition of $\wt{A}^C$ to just concern functions of small error).
\end{proof}

\section{Other dimensions}

\subsection{VC dimension}
\label{sec:VCbounds}

The purpose of this section is proving that
 $\List(\{\pm \}^d) = d$.
 The lower bound on list size follows from Proposition~\ref{instabilitywitnesspf}
 together with the next lemma.

\begin{lemma}\label{vclowerbound}
If $\mc{H}$ has VC dimension $d < \infty$,
then there exists an instability witness of size {$d-1$} for $\mc{H}$
with respect to $\{\pm\}^X$.
 \end{lemma}

 \begin{proof}
{Let $x_0,x_1,\dots,x_{d-1} \in X$} be $d$ points that are shattered by $\mc{H}$. 
Define the witness by $$W(j,b) = (x_j,b)$$ for all $j \in [d-1]$ and $b \in \{\pm\}$.
Item~\eqref{it:real} in the definition of an instability witness is satisfied
with the choice $(x,y) = (x_0,+)$,
because $\mc{H}$ shatters the points
and $x_0$ is not in the image of $W_X$.
To see the second item, 
let $\mc{F}_0$ be all hypotheses in $\{\pm\}^X$ that assign~$+$ to each of $x_1,\ldots,x_{d-1}$, and for $j \in [d-1]$ let $\mc{F}_j$ be all hypotheses that assign~$+$ to each of $x_1,\dots,x_{j-1}$ and assign $-$ to $x_j$.
The families $\mc{F}_0,\ldots,\mc{F}_{d-1}$ indeed partition $\{\pm\}^X$.
It remains to note that
\begin{itemize}

 \item[(a)]  For every $j \in [d-1]$ and $f_0 \in \mc{F}_0$ we have $f_0(W_X(j,-)) =+$.

 \item[(b)] For every $j \in [d-1]$ and $f_j \in \mc{F}_j$ we have $f_j(W_X(j,+)) = -$.
\end{itemize}
 
 \end{proof}

For the upper bound on list size,
think of $\{\pm\}^d$
 as the collection of maps from~$[d]$ to $\{\pm\}$.
Fix $\eps, \rho ,\delta >0$ and let $n$ be a large integer. 
Consider the following algorithm.

\begin{tcolorbox}
\begin{center}
Learner $\mc{A}$ for $\{\pm\}^d$
\end{center}

\ \ \ \
{\bf Input:} A sample $S$ of size $n$.

\medskip

\begin{enumerate}
\item Choose a cutoff $\kappa$ uniformly at random in $[0,\tfrac{\eps}{2d}]$.

\item Let $\hat{P}(i)$ be the fraction of times $i \in [d]$
appeared in $S$.

\item Let $I$ be the set of $i \in [d]$ so that $\hat{P}(i) \geq \kappa$.

\item Output the hypothesis $h$ defined by
\begin{itemize}
\item[(i)] for all $i \in I$, set $h(i)$ to be the label of $i$ in $S$, and
\item[(i)] for all $i \not \in I$, set $h(i)$ to be $+$. 
\end{itemize}

\end{enumerate}
\end{tcolorbox}

For the analysis, fix a realizable distribution $\mc{D}$.
%and denote by $h$ the output of the algorithm. 
%By standard concentration bounds, 
%for $n$ large enough,
%$$\Pr[L_S(h) > \eps] \leq \Pr[ |L_\mc{D}(h) - L_S(h)| > \tfrac{\eps}{2}] \leq \delta.$$
It remains to find a list of $d$ functions that are outputted with high probability
(this list depends on $\mc{D}$).
Without loss of generality,
assume that $\mc{D}(1) \geq \ldots \geq \mc{D}(d)$.
%e)order the points of $X$ so that, excluding $x_+$ (with a $+$), $\mc{D}$ assigns the most probability to $x_1$ (with its sign), the second most to $x_2$, etc., with $x_d$ being assigned the smallest probability. We claim that 
For $j \in [d]$, let $h_j$ be the function that 
assigns the observed sign to $1,\dots,{j}$ and $+$ to $j+1,\dots,d$. 
The list comprises of all $h_j$'s so that $L_\mc{D}(h_j)<\eps$.
By standard concentration bounds,
because $L_S(\mc{A}(S)) \leq \tfrac{\eps}{2}$,
$$\Pr[L_\mc{D}(\mc{A}(S)) > \eps] < \tfrac{\delta}{2}.$$
In addition, a function not in this list is outputted 
only when 
$\hat{P}(i) < \kappa \leq \hat{P}(j)$ for some $i<j$.
If $\mc{D}(i) \leq \mc{D}(j)+ \tfrac{\delta \eps}{8d}$
then $\Pr[\hat{P}(i) < \kappa \leq \hat{P}(j)] \leq \tfrac{\delta}{2}$
because the empirical probabilities are typically close to their true value,
and by the way $\kappa$ was chosen.
Otherwise, $\mc{D}(i) > \mc{D}(j)+ \tfrac{\delta \eps}{8d}$
and then $\Pr[\hat{P}(i) <\hat{P}(j)] \leq \tfrac{\delta}{2}$.
By the union bound,
the chance to output a function that is not in the list
is at most~$\delta$.

\begin{remark*}
{\em The algorithm chooses a random cutoff to ensure global stability.
Roughly speaking, the randomness in the cutoff 
automatically avoids the ``instability $\mc{D}$ hides''.
It is possible to explicitly derandomize this choice
and get a deterministic algorithm,
but the analysis becomes more technical.}
\end{remark*}

\subsection{Littlestone dimension}

In this section we prove the separation between 
list size and the Littlestone dimension.
The separating class is 
the class of $t$-threshold $\mc{T}_t$ for an integer $t>0$.
It comprises of the $t$ functions 
$\tau_i : [t-1] \to \{\pm\}$ for $i \in [t]$ defined by
$\tau_i(x) = +$ iff $x \geq i$.
The list learner for $\mc{T}_t$ is in fact proper.

\begin{theorem}\label{thresholdsize2}
$\List_p(\mc{T}_t) = 2$ for $t >2$.
\end{theorem}

\begin{proof}
The bound $\List(\mc{T}_t) \geq 2$ holds for $t = 3$ (and thus for all larger $t$) by an easy application of the intermediate value theorem.
It remains to prove that $\List_p(\mc{T}_t) \leq 2$ for each $t \ge 3$.
Fix $\eps,\delta > 0$, and take $n$ to be large. 

\begin{tcolorbox}
\begin{center}
Proper learner $\mc{A}$ for $\mc{T}_t$
\end{center}

\ \ \ \
{\bf Input:} A sample $S$ of size $n$.

\medskip

\begin{enumerate}
\item Let $\wh{i}$
be the smallest $i \in [t]$ for which 
$$L_S(\tau_i) < \sigma_i:=\tfrac{\eps}{2(t-i+1)}.$$
\item Output $\tau_{\wh{i}}$.

\end{enumerate}
\end{tcolorbox}

To analyze the algorithm, let
$\mc{D}$ be a realizable distribution. 
Let $i_0$ be the smallest $i \in [t]$ for which $L_\mc{D}(\tau_i) = 0$. 
It follows that $L_\mc{D}(\tau_i)>0$ for $i < i_0$. 
In fact, because $\mc{D}$ is realizable
and relying on the unique structure of $\mc{T}_t$, we have
$$L_\mc{D}(\tau_1) \geq \ldots \geq L_\mc{D}(\tau_{i_0-1})
> L_\mc{D}(\tau_{i_0})=0.$$
We need to identify the list of two function that is outputted with high probability. 
Let $i_*$ be the smallest $i \in [t]$ for which $L_\mc{D}(\tau_i) \le \sigma_i$. 
Keep in mind that $\sigma_i$ is increasing as $i$ increases:
$$\sigma_{i+1} \geq \sigma_{i} +\tfrac{\eps}{10t^2} .$$
The list will always contains $\tau_{i_*}$, which always have small true loss:
$$L_\mc{D}(\tau_{i_*}) \leq \sigma_{i_*} < \eps.$$ 
It follows that
$$\Pr[\wh{i} < i_*-1] < \tfrac{\delta}{4}$$
because for every $i < i_*-1$, 
we have
$$L_\mc{D}(\tau_{i}) \geq L_\mc{D}(\tau_{i_*}-1)
\geq \sigma_{i_*-1} \geq \sigma_i +\tfrac{\eps}{10t^2}.$$
There are three cases we need to consider:

\begin{description}
\item[Case 1: $L_\mc{D}(\tau_{i_*-1}) \le \sigma_{i_*-1}+\tfrac{\eps}{20t^2}$] 
In this case, 
the list comprises of $\tau_{i_*}$ and of $\tau_{i_*-1}$.
The population loss of $\tau_{i_*-1}$ is small as well:
$$L_\mc{D}(\tau_{i_*-1}) \le \sigma_{i_*-1}+\tfrac{\eps}{20t^2} \le \eps.$$
It remains to bound the probability of outputting $\tau_i$ for $i > i_*$.
Because 
$$L_\mc{D}(\tau_{i_*}) 
\leq L_\mc{D}(\tau_{i_*-1}) < \sigma_{i_*-1}+\tfrac{\eps}{20t^2}
\leq \sigma_{i_*}-\tfrac{\eps}{20t^2},$$ 
we have
$$\Pr[\wh{i} > i] \leq \tfrac{\delta}{2}.$$

\item[Case 2: $L_\mc{D}(\tau_{i_*-1}) > \sigma_{i_*-1}+\tfrac{\eps}{20t^2}$
and $L_\mc{D}(\tau_{i_*}) =0$]
In this case, the list comprises only of $\tau_{i_*}$.
By assumption, we have
$$\Pr[\wh{i} = i_*-1] \leq \tfrac{\delta}{4}$$
and
$$\Pr[\wh{i} \leq i_*] \geq \Pr[ L_S(\tau_{i_*})=0] = 1.$$

\item[Case 3: $L_\mc{D}(\tau_{i_*-1}) > \sigma_{i_*-1}+\tfrac{\eps}{20t^2}$
and $L_\mc{D}(\tau_{i_*}) > 0$]
In this case, 
the list comprises of $\tau_{i_*}$ and of $\tau_{i_*+1}$.
As in the previous case,
$$\Pr[\wh{i} < i_*] < \tfrac{\delta}{2}.$$
By assumption, $i_* < i_0$ so that
$$L_\mc{D}(\tau_{i_*+1})
\leq L_\mc{D}(\tau_{i_*}) \leq \sigma_{i_*} 
\leq \sigma_{i_*+1} - \tfrac{\eps}{10t^2} .$$
It follows that $L_\mc{D}(\tau_{i_*+1}) \leq \eps$ and that
$$\Pr[ \wh{i} > i_*+1] \leq \tfrac{\delta}{2}.$$
\end{description}

\end{proof}

\subsection{Hollow star numbers}

\label{sec:HS}

 \begin{theorem}\label{hollowstarlowerbound}
If $\mc{H}$ has hollow star number $s < \infty$,
then there exists an instability witness of size at least $s-2$ for $\mc{H}$
with respect to $\mc{H}$.
 \end{theorem}

 \begin{proof}
Of course, we may assume $s \ge 3$. There are, without loss of generality, points $\{x_0,x_1,\dots,x_{s-1}\}$
such that no hypothesis in $\mc{H}$ assigns $-$ to all points, but for every given $j \in \{0,1,\ldots,s-1\}$, 
there is a hypothesis in $\mc{H}$ that assigns a $+$ only to $x_j$. 
Define the witness by
$$W(j,-) = (x_1,-) \quad \text{and} \quad W(j,+) = (x_{j+1},-)$$ 
for every $j \in [s-2]$.
Item~\eqref{it:real} in the definition is satisfied 
with $(x,y) = (x_0,-)$, because for every $\sigma \in \{\pm \}^{s-2}$,
the image of the map $j \mapsto W_\pm(j, \sigma_j)$ is $\{-\}$,
and the size of the image of $j \mapsto W_X(j,\sigma_j)$ is at most $s-2$.
The second item holds by letting $\mc{F}_0$ denote all hypotheses 
in $\mc{F}$ that assign~$+$ to $x_1$, and 
for $j \in [s-2]$ letting $\mc{F}_j$ denote all hypotheses 
that assign~$-$ to each of $x_{1},\dots,x_{j}$ and assigns $+$ to $x_{j+1}$.  
The sets $\mc{F}_0,\ldots,\mc{F}_{s-2}$  partition~$\mc{H}$ because the all $-$ pattern is missing.
It remains to note that

 \begin{itemize}
 \item[(a)]  For every $j \in [s-1]$ and $f_0 \in \mc{F}_0$ we have
  $f_0(W_X(j,-)) = f_0(x_1) = +$.

 \item[(b)] For every $j \in [s-1]$ and $f_j \in \mc{F}_j$ we have $f_j(W_X(j,+))
 = f_j(x_{j+1}) = +$. \qedhere
 \end{itemize}

 \end{proof}

\bibliographystyle{abbrv}
\bibliography{rep_ref}

\end{document}